\documentclass{article}

\usepackage{arxiv}

\usepackage[utf8]{inputenc} 
\usepackage[T1]{fontenc}    
\usepackage{hyperref}       
\usepackage{url}            
\usepackage{booktabs}       
\usepackage{amsfonts}       
\usepackage{nicefrac}       
\usepackage{microtype}      
\usepackage{lipsum}

\usepackage{amsmath, amsthm, amssymb}
\usepackage{mathtools}
\usepackage{xcolor}
\usepackage{graphicx}
\usepackage{dsfont}
\usepackage{nicefrac}
\usepackage{wasysym}
\usepackage{csquotes}
\usepackage{booktabs}
\usepackage{xspace}
\usepackage{multirow}
\usepackage{float}
\usepackage{array}
\usepackage{bm}

\usepackage{algorithm}
\usepackage[noend]{algorithmic}

\newtheorem{theorem}{Theorem}

\newtheorem{lemma}{Lemma}
\newtheorem{corollary}{Corollary}

\newcommand{\Pois}{\operatorname{Pois}}

\newcommand{\EA}{\text{(1+1)~EA}\xspace}

\newcommand{\EABM}{\text{(1+1)~EA-BM}\xspace}
\newcommand{\EAUM}{\text{(1+1)~EA-UM}\xspace}
\newcommand{\EAMM}{\text{(1+1)~EA-MM}\xspace}

\newcommand{\ie}{i.\,e.\xspace}

\newcommand{\etal}{et al.\ }

\usepackage{tikz}
\usetikzlibrary{positioning, patterns, snakes}
\usepackage{forest}

\title{Runtime Analysis of Evolutionary Algorithms with Biased Mutation for the Multi-Objective Minimum Spanning Tree Problem}

\author{
	Vahid Roostapour \\
	Optimisation and Logistics\\
	The University of Adelaide\\
	Adelaide, Australia \\
	\texttt{vahid.roostapour@adelaide.edu.au} \\
	\And
	Jakob Bossek \\
	Optimisation and Logistics\\
	The University of Adelaide\\
	Adelaide, Australia \\
	\texttt{jakob.bossek@adelaide.edu.au} \\
	\And
	Frank Neumann \\
	Optimisation and Logistics\\
	The University of Adelaide\\
	Adelaide, Australia \\
	\texttt{frank.neumann@adelaide.edu.au} \\
}

\begin{document}
	\maketitle
	
	\begin{abstract}
		Evolutionary algorithms (EAs) are general-purpose problem solvers that usually perform an unbiased search. This is reasonable and desirable in a black-box scenario. For combinatorial optimization problems, often more knowledge about the structure of optimal solutions is given, which can be leveraged by means of biased search operators. We consider the Minimum Spanning Tree (MST) problem in a single- and multi-objective version, and introduce a biased mutation, which puts more emphasis on the selection of edges of low rank in terms of low domination number. We present example graphs where the biased mutation can significantly speed up the expected runtime until (Pareto-)optimal solutions are found. On the other hand, we demonstrate that bias can lead to exponential runtime if \enquote{heavy} edges are necessarily part of an optimal solution. However, on general graphs in the single-objective setting, we show that a combined mutation operator which decides for unbiased or biased edge selection in each step with equal probability exhibits a polynomial upper bound -- as unbiased mutation -- in the worst case and benefits from bias if the circumstances are favorable.
	\end{abstract}

	
	\section{Introduction}
	\label{sec:introduction}
	
	Evolutionary algorithms (EAs) are randomized general-purpose problem solvers that mimic principles from Darwinian evolution theory. These algorithms have proven successful in a wide range of applications, in particular, in tackling (multi-objective) combinatorial $\mathcal{NP}$-hard optimization problems~\cite{DBLP:books/sp/chiong12,Deb2001}. The theoretical understanding of EAs' working principles has made tremendous progress in the past decades with respect to expected runtime analysis, fixed-budget analysis and general convergence aspects~\cite{auger2011theory,DBLP:conf/icec/Rudolph94}.
	
	The problem considered here is a classical combinatorial optimization problem with countless applications in engineering, logistics and many other fields: the \emph{Minimum Spanning Tree} (MST) problem. Given an undirected edge-weighted graph, the goal is to find a spanning sub-graph which is a tree and has minimal total weight among all such trees. When each edge is assigned multiple -- usually conflicting -- weights, one is interested in a set of multi-objective compromise solutions (moMST).
	The single-objective MST problem is well-understood and solvable in polynomial time by well-known algorithms, e.~g., the algorithm by Kruskal~\cite{Kr56}. In contrast, the moMST is proven to be $\mathcal{NP}$-hard~\cite{Ruzika2009} and all deterministic approaches may suffer from potential intractability problems. Here, many successful evolutionary multi-objective algorithms have been proposed (see, e.~g., \cite{ZG1999GeneticAlgorithm, KC2001AComparisonOfEncodings, BG2017AParetoBeneficial}).
	
	In the area of runtime analysis of bio-inspired computation, spanning tree problems have obtained significant attention. The classical MST problem has been investigated for simple single-objective approaches of EAs~\cite{DBLP:journals/tcs/NeumannW07} and ant colony optimization~\cite{DBLP:journals/tcs/NeumannW10}. Furthermore, it has been shown that a multi-objective formulation of the problem can lead to significantly faster evolutionary algorithms~\cite{DBLP:journals/nc/NeumannW06}. For the moMST, it has been shown in \cite{DBLP:journals/eor/Neumann07} that a multi-objective evolutionary algorithm can compute a $2$-approximation in pseudo-polynomial time. Furthermore, the results given in \cite{DBLP:journals/tcs/NeumannW07} have been revisited in the context of multiplicative drift analysis~\cite{DBLP:journals/algorithmica/DoerrJW12} and improved results for special graph classes have been presented in \cite{DBLP:conf/foga/ReichelS09,DBLP:conf/gecco/Witt14}.
	
	Usually evolutionary algorithms perform an unbiased search due to their frequent application in settings where knowledge on the fitness function can only be gained by fitness function evaluations. However, if domain knowledge on the composition of (Pareto-)optimal solutions is available one should incorporate this knowledge into mutation operators to speed up the evolution considerably~\cite{DoerrHN2007,DoerrHN2006,JansenS2010,FriedrichQW2018,FriedrichGQW2018}.
	Neumann and Wegener~\cite{DBLP:journals/eor/Neumann07} introduced an asymmetric mutation operator on bit strings where the probability for a 1-bit to flip depends on the number of 1-bits in the solution and likewise for 0-bits. Here, on average, the number of 1-bits in a solution is not changed which is beneficial for the minimum spanning tree. In fact, the authors were able to obtain runtime speedups adopting this operator for the MST. Jansen and Sudholt~\cite{JansenS2010}
	further investigated this operator. They give examples where asymmetry is beneficial and where it is not. Doerr et al.~\cite{DoerrHN2006, DoerrHN2007} tackle the Eulerian cycle problem with asymmetric mutation and prove much slower runtime bounds in comparison to symmetric mutation. For the classical MST problem it is legitimate to assume that edges of low weight/rank are more likely to be in an MST than edges of high weight/rank. Such knowledge can also be leveraged in terms of biased mutation as demonstrated impressively by Raidl et al.~\cite{DBLP:journals/tec/RaidlKJ06} on random graphs. The authors showed that mutation, where the edge selection probability is biased towards lower rank edges, can lead to immense speedups for evolutionary algorithms for different sub-graph selection problems, inter alia the MST. Likewise, for the moMST problem, non-dominated spanning trees are more likely composed of edges which are dominated by few other edges, i.~e., edges of low non-domination level or domination number. A recent study by Bossek et al.~\cite{DBLP:conf/gecco/BossekG019} confirms this assumption empirically. Both Raidl et al. and Bossek et al. consider the simple edge-exchange mutation on spanning trees: an edge is added to a spanning tree and an edge is dropped from the unique introduced cycle to obtain another spanning tree. Here, the authors introduce bias by modifying the edge selection probability favoring low-rank edges. Both studies serve as a starting point and motivation for our work.
	
	In this paper we consider biased mutation for evolutionary algorithms for the single- and multi-objective MST problem and compare with unbiased counterparts. Specifically, we examine the effects of mutation bias on the time complexity of simple EAs until they hit an optimal solution or cover the Pareto-front for the first time. We show that bias can be both boon and bane depending on the structure of optimal solutions on example graphs. I.~e., there are situations where introduced bias leads to improved upper bounds where we save a factor of $n$ if the ranks of edges which are part of optimal solutions are $O(n)$. Contrarily, if heavy edges are frequent members of optimal solutions, bias towards lightweight edges may entail an exponential deterioration in the expected running time. Luckily, in the single-objective setting, we can combine the best of both worlds. A simple modification, which decides for unbiased or biased mutation in each step independently with probability $1/2$, leads to a guaranteed polynomial runtime bound of $O(n^3\log(n \cdot w_{\max}))$ for general graphs where $w_{\max}$ is the maximum edge weight in the graph. At the same time this strategy benefits from bias if the circumstances allow for it saving on a factor of $n$.
	
	After having motivated our work we introduce the (mo)MST problem formally, establish a vocabulary and introduce the considered algorithms in Section~\ref{sec:preliminaries}. Sections~\ref{sec:single-objective-scenario} and \ref{sec:multi-objective-scenario} deal with our runtime analysis in the single-objective and multi-objective MST setting, respectively. Section~\ref{sec:conclusion} wraps up the work with some concluding remarks and outlook on future work.
	
	\section{Preliminaries}
	\label{sec:preliminaries}
	
	Let $G=(V,E)$ denote a graph with vertex set $V$ and edge set $E$. For convenience, we write $n = |V|$ and $m = |E|$. A spanning tree of graph $G$ is a sub-graph $G' = (V, E')$ if and only if there exists exactly one path between any two vertices in $G'$. In the single-objective scenario, each edge $e \in E$ is assigned a positive weight $w(e)$ and the goal is to find a spanning tree with minimum total weight, called Minimum Spanning Tree (MST). In the multi-objective scenario, each edge is assigned two weights $w(e)=(w_1(e), w_2(e))$.\footnote{Clearly, more than two objective functions are possible. Since we restrict our analysis to bi-objective problems in this paper we refrain from introducing the general form in favor of less notation overhead.} The goal is to find a spanning tree such that the total weight in both weight functions is minimized simultaneously. This may result in a set of incomparable trade-offs which are not necessarily better than each other in both weights. In order to capture this aspect mathematically we adopt the well-known notion of \emph{Pareto dominance}~\cite{CoelloLV2006} -- a core concept in multi-objective optimization -- to establish a partial order of spanning trees. Let $w_i(T) = \sum_{e\in T} w_i(e)$.We say spanning tree $T_1$ weakly (Pareto-)dominates spanning tree $T_2$, denoted by $T_1 \succeq T_2$, if $w_1(T_1)\leq w_1(T_2) \land w_2(T_1)\leq w_2(T_2)$. The strong dominance holds when at least one of the inequalities is strict and it is denoted by $T_1\succ T_2$. $T_1$ is called non-dominated if there is no other spanning tree that dominates $T_1$. Likewise, $w(T_1) = (w_1(T_1), w_2(T_1))$ is the non-dominated objective vector. The union set of all non-dominated spanning trees is called Pareto set, its image in objective space is called the Pareto front, and each solution is termed a Pareto(-optimal) solution or multi-objective MST (moMST). Our goal is to find a non-dominated spanning tree for each non-dominated objective vector.
	In the following, we present the algorithms that we use to tackle these problems.
	
	\subsection{Algorithms}
	\begin{algorithm}[t]
		\caption{\EA}
		\algsetup{indent=1.5em}
		\begin{algorithmic}[1]
			\STATE Let $T$ be a random spanning tree on $G = (V, E)$.
			\STATE Set the edge-selection strategy.
			\WHILE{optimum not found}
			\STATE $T^\prime \leftarrow T$
			\STATE $k\leftarrow 1+\Pois(1)$
			\STATE Based on the selection strategy, assign the probability $q(e)$ to each edge $e\in E$. \label{algline:edge-selection-strategy}
			\FOR{$k$ times}
			\STATE Choose $e\in E$ with probability $q(e)$.
			\STATE $T' \leftarrow T^\prime \cup \{e\}$
			\STATE Drop an edge from the resulting cycle in $T'$ uniformly at random. 
			\ENDFOR
			\IF{ $T^\prime$ has no worse fitness than $T$}
			\STATE $T \leftarrow T^\prime$
			\ENDIF
			\ENDWHILE
		\end{algorithmic}
		\label{alg:oneplusone}
	\end{algorithm}

	We consider the performance of the \EA (see Algorithm~\ref{alg:oneplusone}) facing the single-objective MST problem. It is initialized with a random spanning tree $T$. There have been different studies on generating random spanning trees such as a rather classical randomized algorithm by Broder \cite{DBLP:conf/focs/Broder89}, with expected running time of $O(n\log n)$ for almost all graphs or more recently by Madry et al. \cite{DBLP:conf/soda/MadryST15}. Afterwards, the algorithm sets an \emph{edge-selection strategy}, i.~e., the edge-selection probability distribution that is used in Line \ref{algline:edge-selection-strategy}. Next, the algorithm sets $k=\Pois(1)+1$, the number of edges for the mutation step, where $\Pois(1)$ stems from a Poisson distribution with rate $\lambda = 1$. The constant ensures that we always perform at least one mutation and avoids counting iterations that does not generate new solutions. The same approach has been used in \cite{DBLP:conf/foga/RoostapourP019}. In the mutation step, an edge is selected according to its probability $q(e)$ and is added to $T$. As the mutant is no longer acyclic after the edge insertion, removing a randomly chosen edge from the unique cycle is required to reestablish the tree property. This guarantees that the resulting graph is a spanning tree. The algorithm repeats this procedure $k$ times to achieve a new solution $T'$ and replaces $T$ by $T'$ if $w(T') \leq w(T)$. 
	
	We consider three versions of Algorithm~\ref{alg:oneplusone} where the difference is in the edge-selection strategy.

	\EAUM refers to the unbiased variant of \EA in which always each edge is selected with uniform probability $q(e) = 1/m$. We also consider \EA with biased mutation called \EABM, in which the mutation probability of edge $e$ has been set based on the approximation of the probability that $e$ appears in the MST. The approximation, which is the result of experimental analyses, gives higher probability to the edges with lower weights to be selected. The details on how to calculate the approximation is given in the following sections. Note that for these versions of \EA, the edge-selection strategy does not change during the optimization process and has been set at the beginning of the algorithm. In other words, the edge-selection strategy deterministically assigns values of $q(e)$ (see line~\ref{algline:edge-selection-strategy} in Algorithm~\ref{alg:oneplusone}), i.~e., either uniform or biased mutation with probability 1. Additionally, we analyze a \enquote{hybrid} \EA, called \EAMM (MM for \underline{m}ixed \underline{m}utation), where in each iteration of the outer loop the algorithm decides by fair coin-tossing which strategy (biased or unbiased) to use.

	\begin{algorithm}[t]
		\caption{GSEMO}
		\algsetup{indent=1.5em}
		\begin{algorithmic}[1]
			\STATE Initialize population $P$ with a random spanning tree on $G = (V, E)$.
			\STATE Set the edge-selection strategy.
			\WHILE{not all Pareto-optimal solutions found}
			\STATE Choose $T \in P$ uniformly at random.
			\STATE $T^\prime \leftarrow T$
			\STATE $k\leftarrow 1+\Pois(1)$
			\STATE Based on the selection strategy, assign the probability $q(e)$ to each edge $e\in E$.
			\FOR{$k$ times}
			\STATE Choose $e\in E$ with probability $q(e)$.
			\STATE $T' \leftarrow T^\prime \cup \{e\}$
			\STATE Drop an edge from the resulting cycle in $T'$ uniformly at random. 
			\ENDFOR
			\IF{ $\{ T''\in P \mid T''\succ T'\}=\emptyset$}
			\STATE $P = P\setminus\{T''\in P\mid T'\succeq T''\}\cup\{T'\}$
			\ENDIF
			\ENDWHILE
		\end{algorithmic}
		\label{alg:GSEMO}
	\end{algorithm}
	
	For the multi-objective scenario, our runtime analysis is based on the global simple evolutionary multi-objective algorithm (GSEMO; see Algorithm~\ref{alg:GSEMO}). GSEMO stores a set of non-dominated solutions in the population $P$, which is initialized with a single random spanning tree. In each iteration, it selects a solution $T$ from $P$ uniformly at random and sets the number of the edges to be added in the mutation step: one plus a random value sampled from a Poisson distribution with $\lambda = 1$. The mutation step is the same as the \EA and guarantees that the resulting graph $T'$ is also a spanning tree. If there is no solution in $P$ that strongly dominates $T'$, $T'$ is added to $P$ and all the solutions that $T'$ weakly dominates are removed from $P$. 
	Similar to the single-objective setting, two versions are subject to analysis: GSEMO-UM with uniform edge-selection probability $q(e) = 1/m$ and its biased counterpart GSEMO-BM, in which edges that are dominated by fewer edges in $E$ have higher probability to be selected for the mutation (see Section~\ref{sec:multi-objective-scenario} for details). 

	\section{Single-objective Problem}
	\label{sec:single-objective-scenario}
	
	\begin{figure}[t]
		\centering
		\begin{tikzpicture}[scale=1.8]
        \begin{scope}[every node/.style={circle, draw=transparent, minimum size=0.15em, fill=black}]
            \node (v1) at (0, 0) {};
            \node (v2) at (0.5, 0.5) {};
            \node (v3) at (1, 0) {};
            \node (v4) at (1.5, 0.5) {};
            \node (v5) at (2, 0) {};
            \node (v6) at (2.5, 0) {};
            \node (v7) at (3, 0.5) {};
            \node (v8) at (3.5, 0) {};
        \end{scope}
        \foreach \s/\t/\cost in {
            v1/v2/2a, v1/v3/3a, v3/v2/2a,
            v3/v4/2a, v3/v5/3a, v4/v5/2a,
            v6/v7/2a, v6/v8/3a, v7/v8/2a}
            \draw (\s) edge[thick, sloped, pos=0.5] node[above] {\cost} (\t);
        \draw (v5) edge[dashed, thick] (v6);
        \node[circle, minimum size=5.5em, draw] (kq) at (4, 0) {$G^C$};
        
        \foreach \x/\i in {0.5/1, 1.5/2, 3/p}
            \node at (\x, 0.7) {$T_{\i}$};
    \end{tikzpicture}
		\caption{Triangular-tailed graph $G$ with a chain of $p = n/4$ triangles and a giant component $G^C = K_{n/2}$. \cite{DBLP:journals/tcs/NeumannW07}}
		\label{fig:triangle-graph}
	\end{figure}
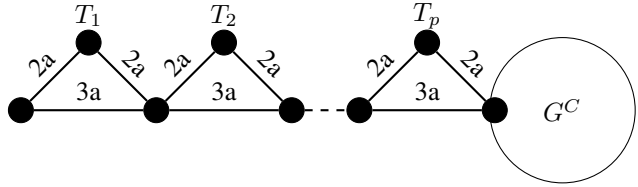
	
	In this section, we consider two types of triangular-tailed graphs, $G_1$ and $G_2$, which are structurally the same but are different in the weights of the edges. A triangular-tailed graph consists of a clique, $G^C$, with $\nu=n/2$ vertices and a triangular tail, $G^T$, with $\eta=n/4$ triangles (Figure~\ref{fig:triangle-graph}). In both $G_1$ and $G_2$, each triangle has 2 edges with weights $2a$ and one edge with weight $3a$, where $a:=n^2$. The weights of edges in the clique are $4a$ and $a$ in $G_1$ and $G_2$, respectively.
	
	Neumann and Wegener proved that \EA with bit-string representation, which flips each bit with probability $1/m$ and is initialized with a random graph, finds the MST of the triangular-tailed graphs in $\Omega(n^4\log{n})$ expected time \cite{DBLP:journals/tcs/NeumannW07}, \ie the triangular-tailed graph has been used as the worst case example to prove the lower bound. This bound is proven for a fitness function that prevents the algorithm to accept solutions other than spanning trees after achieving the first spanning tree. Moreover, the most time consuming phase in their proof is finding the MST from an achieved spanning tree. Hence, their proof also holds even if \EA is initialized with a spanning tree.
	
	Using the same worst case example, we prove that \EAUM finds the MST in $\Theta(n^2\log{n})$. Afterwards, we improve this bound for graph $G_1$, in which the edges of $G^T$ are lighter than the edges of the clique, by enhancing the biased mutation in \EABM. Inspired by the study of Raidl \etal \cite{DBLP:journals/tec/RaidlKJ06}, we use the ranking strategy to perform the biased mutation. To this aim, we assign rank $r$, $1\leq r\leq |E|$, to each edge based on its placement in ascending order of the weights, ties are broken uniformly at random. For each edge $e\in E$ with rank $r$, we approximate the probability of $e$ to appear in the MST with $p(r) = a^r$. Then, we set
	$$
	q(e)=q_{\mathrm{b}}(e)=\frac{\sqrt{p(r)}}{\sum_{i=1}^m \sqrt{p(r)}}\text{,}
	$$
	as the probability of selecting $e$ for the mutation step, where $a=\frac{n-1}{n}$. We show that \EABM finds the MST of $G_1$ in expected time $\Theta(n\log{n})$. However, it takes exponential time for \EABM to find the MST of $G_2$, in which the edges of $G^T$ are heavier than the edges of $G^C$. In the following proofs, let $b = B(T)$ denote the set of bad selected edges in the tail of solution $T$, which have weight $3a$.
	
	\begin{lemma}\label{lem:increase_b}
		\EABM and \EAUM do not increase the value of $b$ during the optimization process.
	\end{lemma}
	\begin{proof}
		Let $T'$ be the result of $k$ subsequent edge insertions into $T$ by mutation. Any changes in the structure of the solution in $G^C$ does not change the weight and neither $b$. It is similar when an edge with weight $2a$ is added and the other edge with weight $2a$ is removed from the cycle. Therefore, we only consider the number of changes in $G^T$ that the swap between $2a$ and $3a$ edges happen in the same triangle. Let $b_{\mathrm{i}}$ and $b_{\mathrm{d}}$ denote the number of swaps that increase and decrease $w(T)$, respectively. We have $|B(T')| = |B(T)| + b_{\mathrm{i}}-b_{\mathrm{d}}$ and $w(T')=w(T)+a(b_{\mathrm{i}}-b_{\mathrm{d}})$. On the other hand, the algorithms accept $T'$ if and only if $w(T')\leq w(T)$, which implies that $b_{\mathrm{i}}\leq b_{\mathrm{d}}$. Thus, in an accepted move, the number of bad edges added to $T$ is less than or equal to the number of added edges with weight $2a$.
	\end{proof}
	
	The following theorem considers the performance of \EAUM on triangular-tailed graphs. 
	\begin{theorem}
		\label{thm:EAUM_runtime_triangular}
		\EAUM finds the MST of triangular-tailed graph $G \in \{G_1,G_2\}$ in $\Theta(n^2\log{n})$ steps with probability $1-o(1)$.
	\end{theorem}
	\begin{proof}
		Here, we follow the proof of Claim 10 in \cite{DBLP:journals/tcs/NeumannW07}. Note that we can focus on $G^T$ since the initial solution is a random spanning tree and all weights in $G^C$ are equal. Moreover, the MST contains all $2a$ edges and no $3a$ edge. Since \EAUM does not increase $b$ (Lemma \ref{lem:increase_b}), we need to calculate the expected time to achieve $b=0$. In order to reduce $b$ by one, the algorithm needs to insert a $2a$ edge and remove the $3a$ edge from the resulting cycle. The probability of adding only one edge is the probability of zero events in the Poisson distribution, which is equal to $e^{-1}$, and there are $b$ specific $2a$ edges that need to be added. Since the maximum size of a consequent cycle is $3$, removing the $3a$ edge happens with the constant probability $1/3$. Hence, the probability of swapping a $2a$ edge with the $3a$ edge in a required triangle is $b/(3em)$ 
		that happens in expected time $3em/b$ by the waiting-time argument. 
		Let $T_{\EAUM}$ denote the first hitting time that \EAUM finds the MST. Since $b$ is at most $n$, we obtain the following upper bound on the expected time with probability $1-o(1)$
		\begin{align*}\label{equ:bound}
		E[T_{\EAUM}] &\leq
		\sum_{k=1}^{n} 3e\cdot\frac{m}{k} \leq 3en^2H_n\\
		&\leq 3en^2(\log n+1) = O(n^2\log n)\text{.}
		\end{align*}
		
		Now we prove the lowe bound.	Similar to the argument in the proof of the coupon collector's theorem (see, e.~g., \cite{motwani_raghavan_1995}), the lower bound $3en^2(\log n+1)-cn^2$ holds with the probability $1-e^{-e^c}$, if \EAUM only adds one edge in each iteration. Setting $c=\frac{\log n}{2}$, the lower bound for the expected time is $\Omega(n^2\log{n})$ with probability $1-o(1)$. 
		Let $k$-step refer to the iterations that $k$ triangle edges are chosen for the mutation step and note that $k\leq3n/4$. It is enough to bound the contribution of $k$-steps on $b$ during $\alpha n^2\log{n}$ iterations for a constant $\alpha>0$.
		The probability of a $k$-step for a constant $k\geq 1$ is $$p^\text{UM}_k=\frac{e^{-1}}{(k-1)!}\cdot{3n/4 \choose k}\cdot \left(\frac{1}{m}\right)^k = \theta(n^k m^{-k}) = \theta(n^{-k})\text{,}$$
		where the first term is the probability of $k-1$ events in the Poisson distribution with $\lambda = 1$. Note that \EAUM always adds at least one edge and $p_0^{\text{UM}} = 0$. Within $\Theta(n^2\log n)$ iterations, the expected number of 2-steps is $O(\log{n})$ and there are $o(1)$ $k$-steps with $k>2$. Each 2-step reduces $b$ by at most 2. On the other hand, in a random spanning tree, each triangle contains a bad edge with probability $2/3$. Thus, $b$ is at least $n/8=\Theta(n)$ with probability $1-e^{-\Omega(n)}$, using a Chernoff bound with $\delta = 2/8$. Hence, with the probability $1-o(1)$, the expected time for \EAUM to find the MST is
		\begin{align*}
		E[T_{\EAUM}] &= \Theta\left( \sum_{k=1}^{b-2\log{n}} \frac{m}{k} \right)=\Theta\left(\sum_{k=1}^{b} \frac{n^2}{k}-\sum_{k=b-2\log{n}}^{b} \frac{n^2}{k} \right)\\
		&= \Theta\left(\sum_{k=1}^{b} \frac{n^2}{k}-\sum_{k=1}^{2\log{n}} \frac{n^2}{k} \right)= \sum_{k=1}^{\Theta(n)} \frac{n^2}{k}-\sum_{k=1}^{O(\log{n})} \frac{n^2}{k}\\
		&= \Theta(n^2\log n)-O(n^2\log\log n) = \Theta (n^2\log n).
		\end{align*}
	\end{proof}
	
	Now, we consider the performance of \EABM on the graphs $G_1$ and $G_2$.
	\begin{lemma}\label{lem:O(n)-probability}
		Using the biased mutation with probability $q_{\mathrm{b}}(e)$, the probability of selecting edge $e$ with rank $r=O(n)$ is $\Theta(1/n)$.
	\end{lemma}
	\begin{proof}
		Considering the denominator of $q_{\mathrm{b}}(e)$, we have
		\begin{align}\nonumber
		\sum_{i=1}^m{a^{i/2}}&=\frac{\sqrt{a}-a^{(m+1)/2}}{1-\sqrt{a}}= \frac{(1-o(1))\cdot\sqrt{a}}{1-\sqrt{a}}= \frac{(1-o(1))\cdot\sqrt{n-1}}{\sqrt{n}-\sqrt{n-1}}=\Theta(n)\text{.}
		\end{align}
		Since $r = O(n)$, for the numerator we have $1\geq(1-\frac{1}{n})^{r/2}\geq (1-\frac{1}{n})^{cn}\geq e^{-c'}$, where $c<c'$ are constants. We conclude that $q_{\mathrm{b}}(e)=\frac{1-o(1)}{2e^{c'}n}=\Theta(1/n)$.
	\end{proof}
	
	Lemma \ref{lem:O(n)-probability} shows that the edges of $G^T$ in $G_1$ are more likely to be chosen in \EABM than in \EAUM. In the following theorem, we show the effect of this property on the performance of \EABM.
	
	\begin{theorem}\label{thm:EABM_runtime_triangular}
		\EABM finds the MST of $G_1$ in $\Theta(n \log{n})$ with probability $1-o(1)$.
	\end{theorem}
	\begin{proof}
		The proof is analogous to the proof of Theorem~\ref{thm:EAUM_runtime_triangular}. However, we use Lemma~\ref{lem:O(n)-probability} to tighten the probability of selecting edges from $G_T$. Hence, the expected waiting for the beneficial event in which a bad edge is removed from the tail is $\Theta(n/b)$. Thus, we obtain an upper bound of $O(n\log n)$. 
		
		To prove the lower bound, similar to the proof of Theorem~\ref{thm:EAUM_runtime_triangular}, we use the argument of coupon collector's theorem with a similar approach used in \cite{DBLP:journals/tcs/DrosteJW02}. However, it must be noted that we argue on the minimum number of edge selections such that all the bad edges are chosen for the mutation at least once. According to Lemma \ref{lem:O(n)-probability}, the probability of selecting an edge in $G^T$ is at least $1/cn$ for a constant $c$. Moreover, we have the initial number of bad edges is at least $n/8$ after the random initialization with probability $1-o(1)$. Note that \EABM selects at least one edge in each iteration.
		
		Therefore, $(1-1/cn)^t$ is the probability of no triangle edge is selected after $t$ iterations. Consequently, the probability of flipping at least one triangle edge in $t$ iterations is $1-(1-1/cn)^t$ that implies $(1-(1-1/cn)^t)^{n/8}$ is the probability of selecting all of the $n/8$ bad edges at least once. Hence, the probability that at least one bad edge has never been selected in $t$ iterations is $1-(1-(1-1/cn)^t)^{n/8}$. Finally, the probability that \EABM does not attempt to remove at least one bad edge in $t=(n-1)\ln{n}$ steps is $1-(1-(1-1/cn)^{(n-1)\ln{n}})^{n/8}\geq 1-e^{-1/8c}$.
		
		Therefore, \EABM needs $\Omega(n\log{n})$ iterations to find the MST with probability of $1-e^{-1/8c}-o(1)$, which completes the proof.
	\end{proof}
	
	Although \EABM efficiently finds the MST of $G_1$, the next argument shows that, in graphs similar to $G_2$, finding the MST takes exponential time.
	
	\begin{lemma}\label{lem:exp_small_probability}
		The probability of selecting an edge with rank $r=\Omega(n^2)$ is exponentially small.
	\end{lemma}
	\begin{proof}
		According to the proof of Lemma \ref{lem:O(n)-probability}, it is enough to show that the enumerator of $p(r)$ is exponentially small when $r>cn^2$ for some constant $c$. To this aim, we have
		$$
		\left(1-\frac{1}{n}\right)^{\frac{r}{2}} \leq \left(1-\frac{1}{n}\right)^{\frac{c}{2}n^2}\leq e^{-\frac{c}{2}n}=O(e^{-n})\text{.}
		$$
	\end{proof}
	\begin{theorem}
		The expected time for \EABM to find the MST of $G_2$ is exponential.
	\end{theorem}
	\begin{proof}
		In $G_2$, edges of $G^T$ have higher weights than the edges of $G^C$. Since there are $\Omega(n^2)$ edges in $G^C$, the rank of edges of $G^T$ is $\Omega(n^2)$. Using the result of Lemma \ref{lem:exp_small_probability}, the probability of selecting any of the edges of $G^T$ is $O(e^{-n})$. Hence, the expected time to select each of these edges for the mutation step is $\Omega(e^n)$. This implies that, in expectation, \EABM needs exponential time to reduce the value of $b$ by one; consequently, it needs exponential time to find the MST of $G_2$.
	\end{proof}
	
	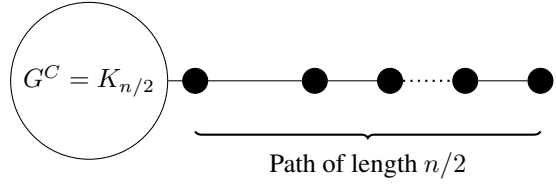
\begin{figure}[t]
		\centering
		\begin{tikzpicture}
    \begin{scope}
        \node[circle, draw, minimum size=5.5em] (Kn) at (0, 0) {$G^C = K_{n/2}$};
        \node[circle, draw, fill = black] (v0)  at (1.41, 0) {};
        \node[circle, draw, fill = black] (v1)  at (3, 0) {};
        \node[circle, draw, fill = black] (v2)  at (4, 0) {};
        \node[circle, draw, fill = black] (vn1) at (5, 0) {};
        \node[circle, draw, fill = black] (vn)  at (6, 0) {};
                
        \draw (Kn) edge[-] node[above] {} (v1);
        \draw (v1) edge[-] node[above] {} (v2);
        \draw (v2) edge[dotted, thick] node[above] {} (vn1);
        \draw (vn1) edge[-] node[above] {} (vn);
        
        \draw [thick, decoration={brace, mirror, raise=0.5cm}, decorate] 
            (v0.south) -- (vn.south) node [pos=0.5, anchor=north ,yshift=-0.65cm] {Path of length $n/2$}; 
    \end{scope}
\end{tikzpicture}
		\caption{Worst case graph for random initialization in the setting of bit-representation.}
		\label{fig:lollypop-graph}
	\end{figure}
	Before we continue with a result on arbitrary graphs we make a short trip into another solution encoding. Let $\mathcal{A}$ refer to the \EA that uses a bit-string representation of the edges instead of spanning trees. Consider the \emph{lollipop graph} presented in Figure~\ref{fig:lollypop-graph} which consists of a clique with $n/2$ vertices and a path of length $n/2$ connected to it. Let all the edges of the clique have lower weights than all the edges of the path. Therefore, the rank of edges in the path is $\Omega(n^2)$ and have $q_b(e)=O(e^{-n})$. Creating a random sub-graph from the lollipop graph, the number of chosen edges from the tail is at most $2n/3$ with probability $1-o(1)$. The lollipop graph illustrates that it is essential for $\mathcal{A}$ to be initialized with a spanning tree. Otherwise, it takes exponential time for it to find even a connected graph.
	
	In the following, we analyze the effect of using both mutation strategies simultaneously in \EAMM. Note that in every $t$ iterations, \EAMM performs $t/3$ uniform mutations and $t-t/3$ biased mutations with probability of $1-o(1)$. This implies that repeating \EAMM $c\geq4$ times, the results of Theorems~\ref{thm:EAUM_runtime_triangular} and \ref{thm:EABM_runtime_triangular} also hold for the \EAMM. However, since \EAMM benefits from the uniform mutation in half of the iterations, it is also able to find the MST of $G_2$ in $\Theta(n^2\log{n})$.
	
	This is the motivation to analyze the performance of \EAMM on general graphs. For arbitrary graph $G$, let $w(T^i)$ be the weight of $T^i$, the spanning tree achieved by the algorithm in iteration $i$, and $T^*$ be the minimum spanning tree. We define $$g(T^i) = w(T^i)-w(T^*)\text{,}$$ the weight gap that the algorithm needs to cover to reach the MST. Note that a MST is not necessarily unique but its weight is unique. We also redefine 1-step as an iteration that the algorithm adds only one edge and removes a random edge from the resulting cycle. Using a similar representation of Lemma 1 in \cite{DBLP:journals/tcs/NeumannW07}, the following lemma presents how 1-steps contribute to reduce the value of $g(T)$. 
	
	\begin{lemma}\label{lem:delta_g}
		Let solution $T$ be an arbitrary spanning tree. There exists a set of $k\in\{1,\cdots,n-1\}$ different accepted 1-steps that if happen in any order transform $T$ to $T^*$ and reduce $w(T)$ by $g(T)/k$ on average.
	\end{lemma}
	\begin{proof}
		Let $E(T)$ and $E(T^*)$ denote the edges of $T$ and $T^*$, respectively. Using an existence proof, Kano \cite{kano1987maximum} proved that there is a bijection $\alpha:E(T^*)\setminus E(T)\rightarrow E(T)\setminus E(T^*)$ such that $w(e)\leq w(\alpha(e))$ and adding $e$ to $T$ creates a cycle that includes $\alpha(e)$. Let $k=|E(T^*)\setminus E(T)|$. Swapping all the edges $e\in E(T^*)\setminus E(T)$ with $\alpha$ transforms $T$ to $T^*$ and reduces $g(T)$ to zero. Thus, each of these good swaps decreases the value of $g(T)$ on average by $g(T)/k$. Moreover, any 1-step that does a good swap is accepted since it results in a solution that is not worse than $T$.
	\end{proof}
	Using the result of Lemma \ref{lem:delta_g} we prove a performance bound on \EAMM on arbitrary graphs.

	\begin{theorem}
		Starting from a random spanning tree, \EAMM finds the minimum spanning tree in expected time $O(n^3\log(n\cdot w_{\max}))$, where $w_{\max}$ is the maximum weight of the edges. 
	\end{theorem}
	\begin{proof}
		Let $\Delta(g)=g(T^{i})-g(T^{i+1})$ be the contribution of the algorithm in reducing the value of $g$ in one iteration. The probability of having a 1-step equals to the probability of having zero events in the Poisson distribution which is $1/e$. Thus, with the probability of $1/(2enm)$, the uniform strategy causes a 1-step such that a specific edge $e$ is added and a specific edge from the created cycle is removed. From Lemma \ref{lem:delta_g} there are $k$ good swaps. Therefore, the probability of a good swap in a 1-step with uniform strategy is $k/(2enm)$. Since a good swap reduces the value of $g(T)$ on average by $g(T)/k$, for $\Delta(g)$ we have
		$$E[\Delta(g)] = \frac{g(T)}{k}\cdot\frac{k}{2enm}.$$
		
		Since the the maximum value of $g(T)$ is $n\cdot w_{\max}$, using the multiplicative drift theorem \cite{DBLP:journals/algorithmica/DoerrG13} with $\delta = 1/(2enm)$, the expected first hitting time that $g(T) =0$ is upper bounded by $$\frac{\ln(n\cdot w_{\max})+1}{\left(\frac{1}{2enm}\right)}= O(n^3\log(n\cdot w_{\max}))\text{.}$$\end{proof}
	Although \EAMM guarantees a polynomial expected time to find $T^*$ for any arbitrary graph, experiments by Raidl et al. showed that in many random graphs, all the edges of $T^*$ have rank $O(n)$. This implies that the expected time for \EAMM to find the MST improves to $O(n^2\log(n\cdot w_{\max}))$ in many applications, since the probability performing a beneficial step improves to $1/(2em)$.

	\section{Multi-Objective Problem}
	\label{sec:multi-objective-scenario}
	In this section we consider the multi-objective version of the minimum spanning tree problem. Firstly, we introduce the ranking of the edges in multi-objective space and experimentally show a considerably good approximation for the appearance of edges in an moMST according to their ranks. Using the approximation, we analyze the performance of GSEMO-UM and GSEMO-BM dealing with two different types of graphs.
	
	\subsection{Experimental Approximation}
	\label{sec:mo-experiments}
	\begin{figure*}
		\centering
		\scalebox{0.91}{
			\input{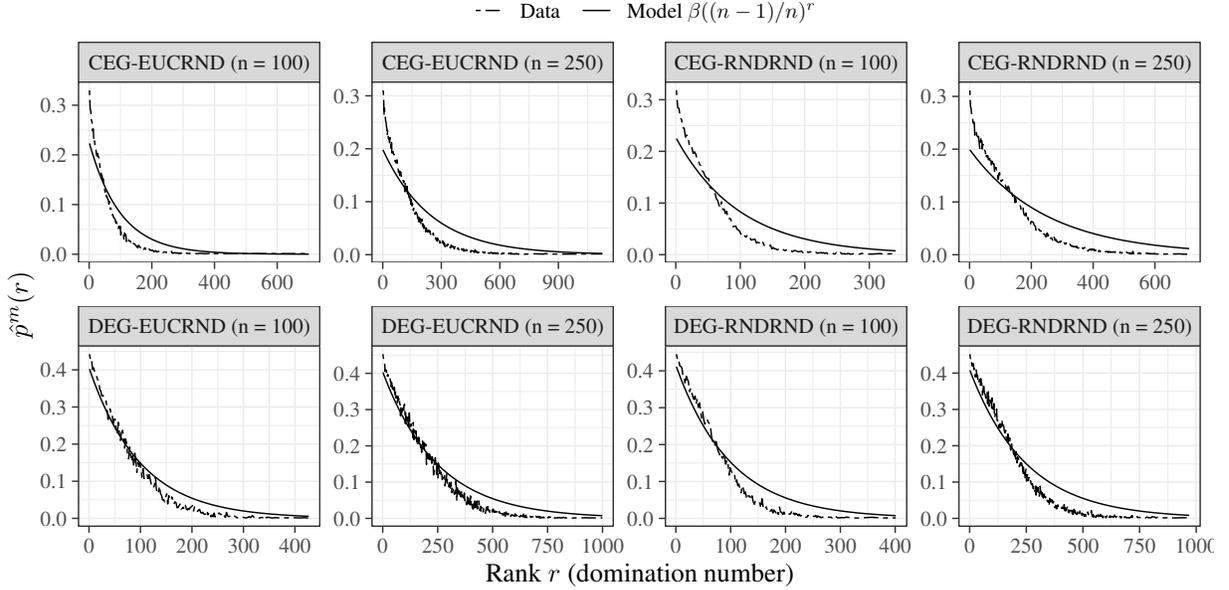}
		} 
		\caption{Empirical probabilities $p^{\mathrm{m}}(r)$ of edges to be part of at least one non-dominated spanning tree as a function of its rank~$r$ measured by the domination number (lower is better). The empirical data is accompanied by regression models of the form $\beta\cdot\left((n-1)/n\right)^r$.}
		\label{fig:gecco2019_mcmst_models}
	\end{figure*}
	
	The work by Raidl. et al.~\cite{DBLP:journals/tec/RaidlKJ06} considered the single-objective scenario and lays the groundwork for our empirical study. As a reminder: the authors showed that low rank edges have a much higher probability to be part of MSTs. 
	In the setting of multiple conflicting objectives similar assumptions are reasonable, i.~e., that non-dominated spanning trees are more likely composed of \enquote{low-rank} edges for an appropriate definition of \enquote{rank}. In a recent study Bossek et al.~\cite{DBLP:conf/gecco/BossekG019} considered different ranking definitions in the bi-objective case. More precisely, they considered (1) the non-domination level and (2) the domination number of an edge to define the rank and established a total order on the edges with low ranks being favored. Similar to Raidl's work, they conducted an empirical study and estimated the probability $p^{\mathrm{m}}(r)$ of edges to be part of at least one spanning tree as a function of its rank $r$ for different graph classes (more details in the following). They, next, empirically evaluated the convergence speed of biased edge selection strategies in comparison to the baseline of random uniform selection. They obtained significant improvements, particularly in the case where the domination number $d(e) = |\{e' \in E\, | \,w(e') \succeq w(e)\}|$ was adopted for the definition of rank and the probability of choosing an edge with edge $r$ for insertion was set to
	\begin{align*}
	q_{\mathrm{b}}^{\mathrm{m}}(r) = \frac{p^{\mathrm{m}}(r)}{\sum_{r} p^{\mathrm{m}}(r)},
	\end{align*}
	i.~e., proportional to its probability of appearance in non-dominated solutions. We catch up on their work and illustrate empirically, that $\beta\cdot\left((n-1)/n\right)^r$ -- similar to Raidl's results -- is indeed a good approximation for the probability $p^{\mathrm{m}}(r)$.
	In line with Bossek et al., our empirical study is based on different graph types reflecting different levels of density and edge weight distribution. Complete graphs (CEG for \underline{C}omplete \underline{E}dge \underline{G}eneration) with $n$ nodes placed uniformly at random in $[0, 100]^2$ are studied alongside graphs where the interconnection of nodes is based on a Delauney triangulation of the point cloud in the Euclidean plane (\underline{D}elauney \underline{E}dge \underline{G}eneration). Note that in the latter case $m = \Theta(n)$. Edge weights $w_i(e), i = 1, 2$ either both are realizations of uniform random numbers stemming from a $\mathcal{U}[5, 200]$-distribution (RNDRND; in consistence with \cite{ZG1999GeneticAlgorithm, KC2001AComparisonOfEncodings}) or the first weight corresponds to the Euclidean distance between the nodes in the plane and the second weight is sampled from a $\mathcal{U}[5, 200]$ distribution (EUCRND). For each graph type, i.~e., CEG-RNDRND, CEG-EUCRND, DEG-RNDRND and DEG-EUCRND we consider $n \in \{25, 50, 100, 250\}$.
	
	The estimation of $p^{\mathrm{m}}(r)$ follows~\cite{DBLP:conf/gecco/BossekG019}. Here, we describe the procedure in a nutshell and refer the interested reader to the original work. First consider a single random graph $G = (V, E)$ of a given graph type and problem size $n$. For each edge $e$, we calculate the number of non-dominated spanning trees that $e$ is part of \footnote{The set of non-dominated spanning trees is approximated by a simple weighted-sum approach minimizing $\lambda w_1(T) + (1-\lambda)w_2(T)$ for equidistantly sampled $\lambda = k/1000, k = 0, \ldots, 1000$.}, termed the share $s(e)$, and estimate the probability of $r$-ranked edges by the average of all shares of the corresponding rank. We repeat this process for $1000$ random graphs of the corresponding graph type and $n \in \{25, 50, 100, 250\}$ and use the mean probability over all $1000$ instances as the final estimate for $p^{\mathrm{m}}(r)$.
	

	Figure~\ref{fig:gecco2019_mcmst_models} shows the estimations of $p^{\mathrm{m}}(r)$, the probability of rank\nobreakdash-$r$ edges to be part of at least one non-dominated spanning tree, separated by graph class and number of nodes. We present results for $n \in \{100, 250\}$ due to space limitations\footnote{Omitted results for $n \in \{25, 50\}$ show the same patterns.}. The estimations are accompanied by fitted regression models of the form $\beta \cdot \left((n-1)/n\right)^r$. We observe that the model mostly adheres quite well to the data. These observations are supported by the results of a regression analysis. Here, the $R^2$ values -- a measure for the fraction of variance in the data explained by the model -- takes values close to 1 with a minimum of $0.8893$ for CEG-EUCRND graphs with $n = 250$ nodes. Additionally, the root mean squared error (RMSE) values, i.~e., the mean deviation of the model predictions to the data, are very low consistently. All in all the experiments support our parametric model assumption for different dense and sparse graphs. As a consequence, we use this empirical estimate for our upcoming theoretical runtime analysis.
	
	\subsection{Theoretical Analysis}
	\begin{figure}[t]
		\centering
		\begin{tikzpicture}[scale=1.8]
        \begin{scope}[every node/.style={circle, draw=transparent, minimum size=0.1em, fill=black}]
            \node (v1) at (0, 0) {};
            \node (v2) at (0.5, 0.5) {};
            \node (v3) at (1, 0) {};
            \node (v4) at (1.5, 0.5) {};
            \node (v5) at (2, 0) {};
            \node (v6) at (2.5, 0) {};
            \node (v7) at (3, 0.5) {};
            \node (v8) at (3.47, 0) {};
            \node (v9) at (3.72,0.15){};
            \node (v10) at (4.0,0.3){};
            \node (v11) at (4.47,0.0){};
            \node (v12) at (4.27,0.35){};
            \node (v13) at (3.73,-0.29){};
            \node (v14) at (4.07,-0.42){};
            \node (v15) at (4.45,-.29){};
        \end{scope}
        \foreach \s/\t/\cost in {
            v1/v2/{1,2} , v1/v3/{2,1} , v3/v2/{1,2} ,
            v3/v4/{1,2} , v3/v5/{2,1} , v4/v5/{1,2} ,
            v6/v7/{1,2} , v6/v8/{2,1} , v7/v8/{1,2} ,
            v9/v10/{$l,l$} , v11/v12/{$l,l$},
            v13/v14/{$l,l$}, v14/v15/{$l,l$}}
            \draw (\s) edge[thick, sloped, pos=0.5] node[above] {\cost} (\t);
        \draw (v5) edge[dashed, thick] (v6) ;
        \node[circle, minimum size=7em, draw] (kq) at (4.16, 0) {$G^C$} ;
        
        \foreach \x/\i in {0.5/1, 1.5/2, 3/p}
            \node at (\x, 0.7) {$T_{\i}$};
    \end{tikzpicture}
		\vspace{-0.3cm}
		\caption{Triangular-tailed graph $G$ with a chain of $p = n/4$ triangles and a giant component $G^C = K_{n/2}$.~\cite{DBLP:journals/tcs/NeumannW07}}
		\label{fig:triangle-graph-mo}
	\end{figure}
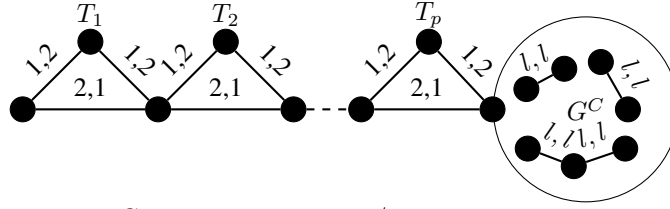
	Motivated by the experimental results, we use $p^{\mathrm{m}}(r)=\beta \cdot \left((n-1)/n\right)^r$ as the approximation for the probability of an edge with domination number $r$ appears in the moMST. As $\beta$ consistently takes values in $(0,1)$ throughout the experiments, we drop this constant factor in subsequent investigations. Note that we break rank ties randomly. Hence, we have $m=|E|$ different edges with $m$ different probabilities. Using Bossek \etal \cite{DBLP:conf/gecco/BossekG019} approach, for each edge $e$ with domination number $r$ we set 
	$$
	q(e)=q_{\mathrm{b}}^{\mathrm{m}}(r)
	$$
	for the probability of choosing $e$ in the mutation step in Algorithm~\ref{alg:GSEMO}. Using the same arguments as in Lemma~\ref{lem:O(n)-probability}, we have the following lemma for biased mutation in the multi-objective setting.
	\begin{lemma}\label{lem:MO-O(n)-probability}
		Using the biased mutation with probability $q(e)=q_{\mathrm{b}}^{\mathrm{m}}(r)$, the probability of selecting edge $e$ with domination count $r=O(n)$ is $\Theta(1/n)$.
	\end{lemma}

	Again, we consider the triangular-tailed graph in two versions $G_1^{\mathrm{m}}$ and $G_2^{\mathrm{m}}$. Both graphs contain $\eta$ triangles in the tail. In each triangle, the two upper edges have weights $(1, 2)$ while the bottom edge has weight $(2, 1)$. The difference lies in the composition of the clique part $G^C$. Here, in $G_1^{\mathrm{m}}$ all edges have the same weight $(k, k)$, $k > 2$ while in $G_2^{\mathrm{m}}$ there exists a subset $G^S= \{e_1, \ldots, e_{l}\}\subseteq G^C $ of size $l\leq(n/2 - 1)$ with $w(e) = (u, u)$, $u>2$, for each edge $e \in G^S$ and $w(e) = (k, k)$, $k > u+n+1$, for all remaining clique edges. We also assume that the edges in $G^S$ do not create any cycle. Let us at this point retain the following: every non-dominated spanning tree of $G_1^{\mathrm{m}}$ contains an arbitrary spanning tree on $G^C$ as a sub-graph. In contrast, in $G_2^{\mathrm{m}}$ every non-dominated spanning tree must necessarily contain $G^S$ as a sub-graph.
	
	Let us briefly state our goals here. We denote by $\mathcal{T}^{*}$ the set of non-dominated spanning trees for a given graph and by $\mathcal{F} = w(\mathcal{T}^{*})$ its image, i.~e., the set of all Pareto-optimal objective vectors.
	We seek to locate for each $f \in \mathcal{F}$ a spanning tree $T^{*} \in \mathcal{T}^{*}$.
	
	\begin{lemma}
		\label{lem:linear_size_pareto_front}
		For both $G_1^{\mathrm{m}}$ and $G_2^{\mathrm{m}}$ we have $|\mathcal{F}| = \Theta(n)$.
	\end{lemma}
	
	\begin{proof}
		Let us first consider the clique part. In $G_1^{\mathrm{m}}$ each spanning tree of $G^C$ has equal weight, we may fix an arbitrary one. In $G_2^{\mathrm{m}}$ the non-dominated spanning tree of $G^C$ must include all the edges of $G^S$. Thus, for each graph type $G_1^{\mathrm{m}}$ and $G_2^{\mathrm{m}}$, the contribution of the edges of $G^C$ in objective values are the same.
		Since the triangular tail is identical for both $G_1^{\mathrm{m}}$ and $G_2^{\mathrm{m}}$, the following observations hold for both versions. Every non-dominated spanning tree contain exactly two edges of each triangle, in particular at least one edge with weight $(1, 2)$, \ie there are at least $\eta$ edges of weight $(1,2)$ in each Pareto solution. Hence, for each non-dominated spanning tree the weight of the triangular part is
		\begin{align*}
		\eta 
		\cdot
		\begin{bmatrix}
		1 \\ 2 \\
		\end{bmatrix}
		+ r \cdot
		\begin{bmatrix}
		1 \\ 2 \\
		\end{bmatrix}
		+ (\eta - r) \cdot
		\begin{bmatrix}
		2 \\ 1 \\
		\end{bmatrix}
		= \begin{bmatrix}
		3\eta - r \\
		3\eta + r \\
		\end{bmatrix}\text{,}
		\end{align*}
		where $0 \leq r \leq \eta$ is the number of triangles that have two upper edges in the spanning tree. Together with our observations in the clique part, this implies $r \in \Theta(n)$ and, as a direct consequence, $|\mathcal{F}| = \Theta(n)$.
	\end{proof}
	
	Let $f_0, f_1, \ldots, f_{p} \in \mathcal{F}$ be the objective vectors in ascending order of the first weight (and thus in descending order of the second weight). In the following, we show that we can easily move between Pareto-optimal spanning trees with distinct weights. We use the notation $d(T, T') := |T \setminus T'|$ and speak about distance of spanning trees in terms of the necessary edge exchange operations needed to transform $T$ to $T'$.
	
	\begin{lemma}\label{lem:connectedness_of_solutions}
		For each non-dominated spanning tree $T$ in $G_1^{\mathrm{m}}$ and $G_2^{\mathrm{m}}$ with $w(T) = f_i, 0 \leq i \leq \eta$, there is a non-dominated spanning tree $T'$ with $d(T, T') = 1$ such that
		\begin{itemize}
			\item $w(T') = f_{i+1}$ for $0 \leq i \leq \eta-1$ or
			\item $w(T') = f_{i-1}$ for $1 \leq i \leq \eta$.
		\end{itemize}
		
	\end{lemma}
	
	\begin{proof}
		
		We only prove the first case. The proof for the other case is similar. Consider a non dominated spanning tree $T$ with $w(T) = f_i, 0 \leq i \leq \eta-1$. $T$ contains exactly $(\eta-i)$ edges of weights $(2, 1)$ in the triangular-tail part. Now we obtain $T'$ by including one of the $\eta-(\eta-i) = i$ remaining edges of weight $(2, 1)$ and dropping a $(1, 2)$ weighted edge on the resulting cycle. It follows that $w(T') = f_{i+1}$ and clearly $d(T, T') = 1$.
	\end{proof}
	
	Lemma~\ref{lem:connectedness_of_solutions} states that once we found a single non-dominated spanning tree it is easy to obtain the others.

	\begin{theorem}
		\label{thm:GSEMO_G_1}
		On $G_1^{\mathrm{m}}$, given an initial spanning tree $T$, GSEMO-UM needs expected time $O(n^3 \log n)$ to cover the Pareto front.
	\end{theorem}
	
	\begin{proof}
		Let $T$ be a spanning tree with $w(T) = f_i$ and $b(T)$ denote the number of triangle edges with weight $(2, 1)$ for $T$. We shall refer those edges \emph{bottom edges} in the following. Since, $w(T) = f_i$ clearly $b(T) = i$. By Lemma~\ref{lem:connectedness_of_solutions} we can move to a tree with weight vector $f_{i+1}$ or $f_{i-1}$ by adding or removing a bottom edge. In GSEMO (see Algorithm~\ref{alg:GSEMO}) achieving $f(i+1)$ happens with probability at least
		\begin{align*}
		\left(\frac{1}{i+1}\right) \cdot \left(e^{-1} \cdot \frac{(\eta-i)}{m}\right) \cdot \left(\frac{2}{3}\right) = \frac{2(\eta-i)}{3em(i+1)}.
		\end{align*}
		Here, the first term is the probability to select the individual $T$ with $w(T) = f_i$ such that $T'$ with $w(T')=f_{i+1}$ is not included in the population yet, the second term is the probability for the $0$ event of a $\text{Pois}(\lambda = 1)$ distribution, \ie, to add exactly one edge to the sampled solution, and the third term is the probability to remove one of the non-bottom edges from the resulting cycle. Adopting waiting time arguments, the expected number of iterations until $f_{i+1}$ is achieved is bounded from above by $3em(i+1)/2(p-i)$. Hence, the total time until the population of GSEMO contains each one solution for each Pareto-optimal objective vector $f_i, i = 0, \ldots, n$ -- only by adding bottom edges and starting with a solution with trade-off $f_0$ in the worst case -- is bounded by the sum 
		\begin{align*}
		\sum_{i=0}^{\eta-1} \frac{3em(i+1)}{2(\eta-i)}
		& = \frac{3em}{2} \cdot \sum_{i=0}^{\eta-1} \frac{(i+1)}{(\eta-i)}
		\leq \frac{3em}{2} \cdot \sum_{i=0}^{\eta-1} \frac{\eta}{(\eta-i)} \\
		& = \frac{3em\eta}{2} \cdot H_\eta = O(n^3 \log n).
		\end{align*}
		
		On the other hand, to include $f_{i-1}$, the algorithm must choose a non-bottom edge from the triangles that include one, with probability $i/(em)$, and remove the bottom edge with probability $1/3$. Thus, the probability of this event is $i/3em(i+1)$, \ie, the expected number of iterations for such event happen is $3em(i+1)/i$. Therefore -- only by decreasing the number of bottom edges and starting from $f_{\eta}$ in the worst case -- the expected time for GSEMO to achieve all the objective vectors in the Pareto front is upper bounded by the sum
		
		\begin{align*}
		\sum_{i=1}^{\eta} \frac{3em(i+1)}{i}
		& = 3em \cdot \sum_{i=1}^{\eta} \frac{(i+1)}{i}
		\leq 3em \cdot \sum_{i=1}^{\eta} \frac{\eta}{i} \\
		& = 3em\eta \cdot H_\eta = O(n^3 \log n).
		\end{align*}
		
		All together, since one of the cases is always available, the total upper bound is $O(n^3\log{n})$. 
	\end{proof}
	
	Next we consider the performance of GSEMO-UM on $G_2^{\mathrm{m}}$. In an arbitrary moMST $T$ of $G_2^{\mathrm{m}}$, let $s=|G^S\cap T|$ denote the number of optimal edges $G^S$ in $T$.
	\begin{lemma} \label{lem:strict_dominance}
		For two solutions $T_1,T_2\in G_2^{\mathrm{m}}$, $T_1\succ T_2$ if and only if $s_1>s_2$.
	\end{lemma}
	\begin{proof}
		Considering the proof of lemma \ref{lem:linear_size_pareto_front}, the difference between the objective values of $T_1$ and $T_2$ that is caused because of the chosen tail edges is at most $n$. On the other hand, increasing $s$ improves both objective values by at least $n+1$. Thus any solution that has larger $s$ have strictly better objective value in both objectives.
	\end{proof}
	Lemma \ref{lem:strict_dominance} results in the fact that all the solutions in the population set of GSEMO have the same value of $s$.
	Note that GSEMO starts with a spanning tree and any offspring is also a spanning tree.
	
	\begin{theorem}\label{thm:GSEMO_G_2}
		On $G_2^{\mathrm{m}}$, given an initial spanning tree $T$, GSEMO-UM needs expected time $O(n^3\log{n})$ to cover the Pareto-front.
	\end{theorem}
	\begin{proof}
		We consider two phases in this proof. The first phase is to find the solution $T$ with $s=l$, \ie $T$ contains all the edges of $G^S$. After this phase, we know that any offspring is a Pareto-optimal solution. The next phase is to cover the whole Pareto front.
		
		Note that all the solutions in the current population have the same value of $s<l$. Thus, the probability of choosing solution $T$ with highest $s$ is 1. To increase $s$, the algorithm adds edge $e \in G^S\setminus T$ with probability $(l-s)/m$. Adding $e$ can cause a cycle with size at most $n$. In the worst case, there is only one edge in the cycle that can be removed without removing another optimal edge. Hence, a beneficial removing happens with probability of $1/n$. Therefore, the probability of increasing $s$ by $1$ is at least $e^{-1}\cdot\frac{l-s}{m}\cdot \frac{1}{n}\text{,}$ where $e^{-1}$ is the probability that GSEMO adds only one edge. Such mutation step happens after $O(mn/(l-s))$ iterations in expectation. The minimum initial value for $s$ is zero and $l$ is at most $n-1$. Thus, the expected time for GSEMO to finish phase one is upper bounded by
		$$\sum_{i=0}^{l-1} \frac{mn}{e(l-s)}\leq n^3 \sum_{i=1}^{n} \frac{1}{n} = O(n^3\log{n})$$
		
		In the second phase, GSEMO does not accept a solution with $s<l$. Hence, the same argument as in Theorem \ref{thm:GSEMO_G_1} proves that GSEMO finishes the second phase in $O(n^3\log{n})$ expected time and this completes the proof.
	\end{proof}
	
	Now we consider GSEMO-BM algorithm with biased mutation that select the edges with $q(e)=q_{\mathrm{b}}^{\mathrm{m}}(r)\text{.}$ The number of edges in the tail of $G_1^{\mathrm{m}}$ and $G_2^{\mathrm{m}}$ is the same and equal to $3n/4$. In both of the graphs, these edges dominate every other edges and consequently have lower non-domination ranks, \ie each edge have a unique random rank within $\{1,\cdots,3n/4\}$. Moreover, in $G_2^{\mathrm{m}}$, edges of $G^S$ dominate other edges of $G^C$. Hence, ranks $3n/4, \cdots,(3n/4)+l$ belong to the edges of $G^S$. Therefore, as Lemma \ref{lem:MO-O(n)-probability} shows, all the edges that belong to the moMSTs in $G_1^{\mathrm{m}}$ and $G_2^{\mathrm{m}}$ have the selection probability $\Theta(1/n)$. Using the same arguments as in Theorems \ref{thm:GSEMO_G_1} and \ref{thm:GSEMO_G_2}, the following result hold for the performance of GSEMO on the graphs $G_1^{\mathrm{m}}$ and $G_2^{\mathrm{m}}$.
	\begin{corollary}
		On $G_1^{\mathrm{m}}$ and $G_2^{\mathrm{m}}$, given an initial spanning tree $T$, GSEMO-BM needs expected time $O(n^2 \log n)$ to cover the Pareto-front.
	\end{corollary}
	
	\section{Conclusion}
	\label{sec:conclusion}
	
	We performed a rigorous asymptotic runtime analysis of evolutionary algorithms with biased mutation for the classic Minimum Spanning Tree problem. Bias in this context means that edges of low weight in the single-objective case and of low domination number in the multi-objective case are assigned a higher probability of mutation. Our findings reveal that bias is blessing and curse at the same time. While a significant time complexity speedup can be achieved in some cases, bias may also lead to exponential expected optimization time if edges of high rank are part of optimal solutions. We showed that using the biased and unbiased mutations simultaneously is the key to avoid the extreme cases of bias. We will consider the generalization of the achieved results to more general graph classes in future work.
	
	\section*{Acknowledgment}
	This work has been supported by the Australian Research Council (ARC) through grants DP160102401 and DP190103894.

	\bibliographystyle{unsrt}  
	\bibliography{references}  

\begin{thebibliography}{10}

\bibitem{DBLP:books/sp/chiong12}
Raymond Chiong, Thomas Weise, and Zbigniew Michalewicz.
\newblock {\em Variants of Evolutionary Algorithms for Real-World
  Applications}.
\newblock Springer Publishing Company, Incorporated, 2011.

\bibitem{Deb2001}
Kalyanmoy Deb.
\newblock {\em Multi-Objective Optimization Using Evolutionary Algorithms}.
\newblock John Wiley \& Sons, Inc., New York, NY, USA, 2001.

\bibitem{auger2011theory}
Anne Auger and Benjamin Doerr.
\newblock {\em Theory of randomized search heuristics: Foundations and recent
  developments}, volume~1.
\newblock World Scientific, 2011.

\bibitem{DBLP:conf/icec/Rudolph94}
G{\"{u}}nter Rudolph.
\newblock Convergence of non-elitist strategies.
\newblock In {\em Proceedings of the First {IEEE} Conference on Evolutionary
  Computation, {IEEE} World Congress on Computational Intelligence, Orlando,
  Florida, USA, June 27-29, 1994}, pages 63--66. {IEEE}, 1994.

\bibitem{Kr56}
Joseph~B. Kruskal.
\newblock On the shortest spanning subtree of a graph and the traveling
  salesman problem.
\newblock {\em Proceedings of the American Mathematical Society}, 7(1):48--50,
  1956.

\bibitem{Ruzika2009}
Stefan Ruzika and Horst~W. Hamacher.
\newblock {A survey on multiple objective minimum spanning tree problems}.
\newblock {\em Lecture Notes in Computer Science (including subseries Lecture
  Notes in Artificial Intelligence and Lecture Notes in Bioinformatics)}, 5515
  LNCS:104--116, 2009.

\bibitem{ZG1999GeneticAlgorithm}
G~Zhou and M~Gen.
\newblock {Genetic Algorithm Approach on Multi-Criteria Minimum Spanning Tree
  Problem}.
\newblock {\em European Journal of Operational Research}, 114:141--152, 1999.

\bibitem{KC2001AComparisonOfEncodings}
J.D. Knowles and D.W. Corne.
\newblock {A Comparison of Encodings and Algorithms for Multiobjective Minimum
  Spanning Tree Problems}.
\newblock {\em Evolutionary Computation}, 1:544--551, 2001.

\bibitem{BG2017AParetoBeneficial}
J~Bossek and C~Grimme.
\newblock A pareto-beneficial sub-tree mutation for the multi-criteria minimum
  spanning tree problem.
\newblock In {\em Proceedings of the 2017 IEEE Symposium Series on
  Computational Intelligence (SSCI)}, pages 3280--3287, Honolulu, Hawai, 2017.
  IEEE.

\bibitem{DBLP:journals/tcs/NeumannW07}
Frank Neumann and Ingo Wegener.
\newblock Randomized local search, evolutionary algorithms, and the minimum
  spanning tree problem.
\newblock {\em Theoretical Computer Science}, 378(1):32--40, 2007.

\bibitem{DBLP:journals/tcs/NeumannW10}
Frank Neumann and Carsten Witt.
\newblock Ant colony optimization and the minimum spanning tree problem.
\newblock {\em Theoretical Computer Science}, 411(25):2406--2413, 2010.

\bibitem{DBLP:journals/nc/NeumannW06}
Frank Neumann and Ingo Wegener.
\newblock Minimum spanning trees made easier via multi-objective optimization.
\newblock {\em Natural Computing}, 5(3):305--319, 2006.

\bibitem{DBLP:journals/eor/Neumann07}
Frank Neumann.
\newblock Expected runtimes of a simple evolutionary algorithm for the
  multi-objective minimum spanning tree problem.
\newblock {\em European Journal of Operational Research}, 181(3):1620--1629,
  2007.

\bibitem{DBLP:journals/algorithmica/DoerrJW12}
Benjamin Doerr, Daniel Johannsen, and Carola Winzen.
\newblock Multiplicative drift analysis.
\newblock {\em Algorithmica}, 64(4):673--697, 2012.

\bibitem{DBLP:conf/foga/ReichelS09}
Joachim Reichel and Martin Skutella.
\newblock On the size of weights in randomized search heuristics.
\newblock In {\em Foundations of Genetic Algorithms, 10th {ACM} {SIGEVO}
  International Workshop, {FOGA} 2009, Orlando, Forida, USA, January 9-11,
  2009, Proceedings}, pages 21--28, 2009.

\bibitem{DBLP:conf/gecco/Witt14}
Carsten Witt.
\newblock Revised analysis of the {(1+1)} ea for the minimum spanning tree
  problem.
\newblock In {\em Genetic and Evolutionary Computation Conference, {GECCO} '14,
  Vancouver, BC, Canada, July 12-16, 2014}, pages 509--516, 2014.

\bibitem{DoerrHN2007}
Benjamin Doerr, Nils Hebbinghaus, and Frank Neumann.
\newblock {Speeding Up Evolutionary Algorithms Through Asymmetric Mutation
  Operators}.
\newblock {\em Evolutionary Computation}, 15(4):401--410, dec 2007.

\bibitem{DoerrHN2006}
Benjamin Doerr, Nils Hebbinghaus, and Frank Neumann.
\newblock {Speeding Up Evolutionary Algorithms Through Restricted Mutation
  Operators}.
\newblock In Thomas~Philip Runarsson, Hans-Georg Beyer, Edmund Burke, Juan~J
  Merelo-Guerv{\'{o}}s, L~Darrell Whitley, and Xin Yao, editors, {\em Parallel
  Problem Solving from Nature - PPSN IX}, pages 978--987, Berlin, Heidelberg,
  2006. Springer Berlin Heidelberg.

\bibitem{JansenS2010}
Thomas Jansen and Dirk Sudholt.
\newblock {Analysis of an Asymmetric Mutation Operator}.
\newblock {\em Evolutionary Computation}, 18(1):1--26, mar 2010.

\bibitem{FriedrichQW2018}
Tobias Friedrich, Francesco Quinzan, and Markus Wagner.
\newblock {Escaping Large Deceptive Basins of Attraction with Heavy-tailed
  Mutation Operators}.
\newblock In {\em Proceedings of the Genetic and Evolutionary Computation
  Conference}, GECCO '18, pages 293--300, New York, NY, USA, 2018. ACM.

\bibitem{FriedrichGQW2018}
Tobias Friedrich, Andreas G{\"{o}}bel, Francesco Quinzan, and Markus Wagner.
\newblock {Heavy-Tailed Mutation Operators in Single-Objective Combinatorial
  Optimization}.
\newblock In Anne Auger, Carlos~M Fonseca, Nuno Louren{\c{c}}o, Penousal
  Machado, Lu{\'{i}}s Paquete, and Darrell Whitley, editors, {\em Parallel
  Problem Solving from Nature -- PPSN XV}, pages 134--145, Cham, 2018. Springer
  International Publishing.

\bibitem{DBLP:journals/tec/RaidlKJ06}
G{\"{u}}nther~R. Raidl, Gabriele Koller, and Bryant~A. Julstrom.
\newblock Biased mutation operators for subgraph-selection problems.
\newblock {\em {IEEE} Trans. Evolutionary Computation}, 10(2):145--156, 2006.

\bibitem{DBLP:conf/gecco/BossekG019}
Jakob Bossek, Christian Grimme, and Frank Neumann.
\newblock On the benefits of biased edge-exchange mutation for the
  multi-criteria spanning tree problem.
\newblock In Anne Auger and Thomas St{\"{u}}tzle, editors, {\em Proceedings of
  the Genetic and Evolutionary Computation Conference, {GECCO} 2019, Prague,
  Czech Republic, July 13-17, 2019}, pages 516--523. {ACM}, 2019.

\bibitem{CoelloLV2006}
Carlos A~Coello Coello, Gary~B Lamont, and David A~Van Veldhuizen.
\newblock {\em {Evolutionary Algorithms for Solving Multi-Objective Problems
  (Genetic and Evolutionary Computation)}}.
\newblock Springer-Verlag New York, Inc., Secaucus, NJ, USA, 2006.

\bibitem{DBLP:conf/focs/Broder89}
Andrei~Z. Broder.
\newblock Generating random spanning trees.
\newblock In {\em 30th Annual Symposium on Foundations of Computer Science,
  Research Triangle Park, North Carolina, USA, 30 October - 1 November 1989},
  pages 442--447. {IEEE} Computer Society, 1989.

\bibitem{DBLP:conf/soda/MadryST15}
Aleksander Madry, Damian Straszak, and Jakub Tarnawski.
\newblock Fast generation of random spanning trees and the effective resistance
  metric.
\newblock In Piotr Indyk, editor, {\em Proceedings of the Twenty-Sixth Annual
  {ACM-SIAM} Symposium on Discrete Algorithms, {SODA} 2015, San Diego, CA, USA,
  January 4-6, 2015}, pages 2019--2036. {SIAM}, 2015.

\bibitem{DBLP:conf/foga/RoostapourP019}
Vahid Roostapour, Mojgan Pourhassan, and Frank Neumann.
\newblock Analysis of baseline evolutionary algorithms for the packing while
  travelling problem.
\newblock In Tobias Friedrich, Carola Doerr, and Dirk~V. Arnold, editors, {\em
  Proceedings of the 15th {ACM/SIGEVO} Conference on Foundations of Genetic
  Algorithms, {FOGA} 2019, Potsdam, Germany, August 27-29, 2019}, pages
  124--132. {ACM}, 2019.

\bibitem{motwani_raghavan_1995}
Rajeev Motwani and Prabhakar Raghavan.
\newblock {\em Randomized Algorithms}.
\newblock Cambridge University Press, 1995.
\newblock Page 61.

\bibitem{DBLP:journals/tcs/DrosteJW02}
Stefan Droste, Thomas Jansen, and Ingo Wegener.
\newblock On the analysis of the {(1+1)} evolutionary algorithm.
\newblock {\em Theor. Comput. Sci.}, 276(1-2):51--81, 2002.

\bibitem{kano1987maximum}
Mikio Kano.
\newblock Maximum and $k$-th maximal spanning trees of a weighted graph.
\newblock {\em Combinatorica}, 7(2):205--214, 1987.

\bibitem{DBLP:journals/algorithmica/DoerrG13}
Benjamin Doerr and Leslie~Ann Goldberg.
\newblock Adaptive drift analysis.
\newblock {\em Algorithmica}, 65(1):224--250, 2013.

\end{thebibliography}

\end{document}